\documentclass[twoside]{article}

\usepackage[accepted]{aistats2014}
\usepackage{amsthm}
\usepackage{amssymb}
\usepackage{graphicx}
\usepackage{times}
\usepackage{subfigure}
\usepackage{color}
\usepackage{amsbsy}
\usepackage{amsmath}
\usepackage{amsfonts}
\usepackage{epsfig}
\usepackage{wrapfig}
\usepackage{tabulary}

\begin{document}

%

%

\twocolumn[

\aistatstitle{Selective Sampling with Drift}

\aistatsauthor{Edward Moroshko \And Koby Crammer}

\aistatsaddress{ Department of Electrical Engineering,\\ The Technion, Haifa, Israel\\ \texttt{edward.moroshko@gmail.com} \And Department of Electrical Engineering,\\ The Technion, Haifa, Israel\\ \texttt{koby@ee.technion.ac.il}} ]

\begin{abstract}
Recently there has been much work on selective sampling, an online
active learning setting, in which algorithms work in rounds. On each
round an algorithm receives an input and makes a prediction. Then, it
can decide whether to query a label, and if so to update its model,
otherwise the input is discarded. Most of this work is focused on the
stationary case, where it is assumed that there is a fixed target
model, and the performance of the algorithm is compared to a fixed
model. However, in many real-world applications, such as spam
prediction, the best target function may drift over time, or have
shifts from time to time. 
We develop a novel selective sampling
algorithm for the drifting setting, 
analyze it 
under no assumptions on the mechanism generating the sequence
of instances, and derive new mistake bounds that depend on the amount
of drift in the problem. Simulations on synthetic and real-world
datasets demonstrate the superiority of our algorithms as a selective
sampling algorithm in the drifting setting.

\end{abstract}

%

\newtheorem{theorem}{Theorem}
\newtheorem{lemma}[theorem]{Lemma}
\newtheorem{definition}[theorem]{Definition}
\newtheorem{claim}[theorem]{Claim}
\newtheorem{corollary}[theorem]{Corollary}
\newtheorem{remark}[theorem]{Remark}

\def\proofsketch{\par\penalty-1000\vskip .5 pt\noindent{\bf Proof sketch\/: }}
\def\ProofSketch{\par\penalty-1000\vskip .1 pt\noindent{\bf Proof sketch\/: }}
\newcommand{\QED}{\hfill$\;\;\;\rule[0.1mm]{2mm}{2mm}$\\}
\renewcommand{\qedsymbol}{}

\newcommand{\todo}[1]{{~\\\bf TODO: {#1}}~\\}

\newfont{\msym}{msbm10}
\newcommand{\reals}{\mathbb{R}}
\newcommand{\half}{\frac{1}{2}}
\newcommand{\sign}{{\rm sign}}
\newcommand{\paren}[1]{\left({#1}\right)}
\newcommand{\brackets}[1]{\left[{#1}\right]}
\newcommand{\braces}[1]{\left\{{#1}\right\}}
\newcommand{\ceiling}[1]{\left\lceil{#1}\right\rceil}
\newcommand{\abs}[1]{\left\vert{#1}\right\vert}
\newcommand{\tr}{{\rm Tr}}
\newcommand{\pr}[1]{{\rm Pr}\left[{#1}\right]}
\newcommand{\prp}[2]{{\rm Pr}_{#1}\left[{#2}\right]}
\newcommand{\Exp}[1]{{\rm E}\left[{#1}\right]}
\newcommand{\Expp}[2]{{\rm E}_{#1}\left[{#2}\right]}
\newcommand{\eqdef}{\stackrel{\rm def}{=}}
\newcommand{\comdots}{, \ldots ,}
\newcommand{\true}{\texttt{True}}
\newcommand{\false}{\texttt{False}}
\newcommand{\mcal}[1]{{\mathcal{#1}}}
\newcommand{\argmin}[1]{\underset{#1}{\mathrm{argmin}} \:}
\newcommand{\normt}[1]{\left\Vert {#1} \right\Vert^2}
\newcommand{\step}[1]{\left[#1\right]_+}
\newcommand{\1}[1]{[\![{#1}]\!]}
\newcommand{\diag}{{\textrm{diag}}}
\newcommand{\KL}{{\textrm{D}_{\textrm{KL}}}}
\newcommand{\IS}{{\textrm{D}_{\textrm{IS}}}}
\newcommand{\EU}{{\textrm{D}_{\textrm{EU}}}}

\newcommand{\leftmarginpar}[1]{\marginpar[#1]{}}
\newcommand{\figline}{\rule{0.45\textwidth}{0.5pt}}
\newcommand{\figlinee}[1]{\rule{#1\textwidth}{0.5pt}}
\newcommand{\pseudocodefont}{\normalsize}
\newcommand{\nolineskips}{
\setlength{\parskip}{0pt}
\setlength{\parsep}{0pt}
\setlength{\topsep}{0pt}
\setlength{\partopsep}{0pt}
\setlength{\itemsep}{0pt}}

\newcommand{\beq}[1]{\begin{equation}\label{#1}}
\newcommand{\eeq}{\end{equation}}
\newcommand{\beqa}{\begin{eqnarray}}
\newcommand{\eeqa}{\end{eqnarray}}
\newcommand{\secref}[1]{Section~\ref{#1}}
\newcommand{\figref}[1]{Fig.~\ref{#1}}
\newcommand{\exmref}[1]{Example~\ref{#1}}
\newcommand{\thmref}[1]{Thm.~\ref{#1}}
\newcommand{\sthmref}[1]{Thm.~\ref{#1}}
\newcommand{\defref}[1]{Definition~\ref{#1}}
\newcommand{\remref}[1]{Remark~\ref{#1}}
\newcommand{\chapref}[1]{Chapter~\ref{#1}}
\newcommand{\appref}[1]{App.~\ref{#1}}
\newcommand{\lemref}[1]{Lem.~\ref{#1}}
\newcommand{\propref}[1]{Proposition~\ref{#1}}
\newcommand{\claimref}[1]{Claim~\ref{#1}}

\newcommand{\corref}[1]{Corollary~\ref{#1}}
\newcommand{\scorref}[1]{Cor.~\ref{#1}}
\newcommand{\tabref}[1]{Table~\ref{#1}}
\newcommand{\tran}[1]{{#1}^{\top}}
\newcommand{\norm}{\mcal{N}}
\newcommand{\eqsref}[1]{Eqns.~(\ref{#1})}

\newcommand{\mb}[1]{{\boldsymbol{#1}}}
\newcommand{\up}[2]{{#1}^{#2}}
\newcommand{\dn}[2]{{#1}_{#2}}
\newcommand{\du}[3]{{#1}_{#2}^{#3}}
\newcommand{\textl}[2]{{$\textrm{#1}_{\textrm{#2}}$}}

\newcommand{\lf}{\lambda_{F}}

\newcommand{\vx}{\mathbf{x}}
\newcommand{\vxi}[1]{\vx_{#1}}
\newcommand{\vxii}{\vxi{t}}

\newcommand{\yi}[1]{y_{#1}}
\newcommand{\yii}{\yi{t}}
\newcommand{\hyi}[1]{\hat{y}_{#1}}
\newcommand{\hyii}{\hyi{i}}

\newcommand{\vy}{\mb{y}}
\newcommand{\vyi}[1]{\vy_{#1}}
\newcommand{\vyii}{\vyi{i}}

\newcommand{\vn}{\mb{\nu}}
\newcommand{\vni}[1]{\vn_{#1}}
\newcommand{\vnii}{\vni{i}}

\newcommand{\vmu}{\mb{\mu}}
\newcommand{\vmus}{{\vmu^*}}
\newcommand{\vmuts}{{\vmus}^{\top}}
\newcommand{\vmui}[1]{\vmu_{#1}}
\newcommand{\vmuii}{\vmui{i}}

\newcommand{\vmut}{\vmu^{\top}}
\newcommand{\vmuti}[1]{\vmut_{#1}}
\newcommand{\vmutii}{\vmuti{i}}

\newcommand{\vsigma}{\mb \sigma}
\newcommand{\msigma}{\Sigma}
\newcommand{\msigmas}{{\msigma^*}}
\newcommand{\msigmai}[1]{\msigma_{#1}}
\newcommand{\msigmaii}{\msigmai{t}}

\newcommand{\mups}{\Upsilon}
\newcommand{\mupss}{{\mups^*}}
\newcommand{\mupsi}[1]{\mups_{#1}}
\newcommand{\mupsii}{\mupsi{i}}
\newcommand{\upssl}{\upsilon^*_l}

\newcommand{\vu}{\mathbf{u}}
\newcommand{\vut}{\tran{\vu}}
\newcommand{\vui}[1]{\vu_{#1}}
\newcommand{\vuti}[1]{\vut_{#1}}
\newcommand{\hvu}{\hat{\vu}}
\newcommand{\hvut}{\tran{\hvu}}
\newcommand{\hvur}[1]{\hvu_{#1}}
\newcommand{\hvutr}[1]{\hvut_{#1}}
\newcommand{\vw}{\mb{w}}
\newcommand{\vwi}[1]{\vw_{#1}}
\newcommand{\vwii}{\vwi{t}}
\newcommand{\vwti}[1]{\vwt_{#1}}
\newcommand{\vwt}{\tran{\vw}}

\newcommand{\tvw}{\tilde{\mb{w}}}
\newcommand{\tvwi}[1]{\tvw_{#1}}
\newcommand{\tvwii}{\tvwi{t}}

\newcommand{\vv}{\mb{v}}
\newcommand{\vvt}{\tran{\vv}}

\newcommand{\vvi}[1]{\vv_{#1}}
\newcommand{\vvti}[1]{\vvt_{#1}}
\newcommand{\lambdai}[1]{\lambda_{#1}}
\newcommand{\Lambdai}[1]{\Lambda_{#1}}

\newcommand{\vxt}{\tran{\vx}}
\newcommand{\hvx}{\hat{\vx}}
\newcommand{\hvxi}[1]{\hvx_{#1}}
\newcommand{\hvxii}{\hvxi{i}}
\newcommand{\hvxt}{\tran{\hvx}}
\newcommand{\hvxti}[1]{\hvxt_{#1}}
\newcommand{\hvxtii}{\hvxti{i}}
\newcommand{\vxti}[1]{\vxt_{#1}}
\newcommand{\vxtii}{\vxti{i}}

\newcommand{\vb}{\mb{b}}
\newcommand{\vbt}{\tran{\vb}}
\newcommand{\vbi}[1]{\vb_{#1}}

\newcommand{\hvy}{\hat{\vy}}
\newcommand{\hvyi}[1]{\hvy_{#1}}


\renewcommand{\mp}{P}
\newcommand{\mpd}{\mp^{(d)}}
\newcommand{\mpt}{\mp^T}
\newcommand{\tmp}{\tilde{\mp}}
\newcommand{\mpi}[1]{\mp_{#1}}
\newcommand{\mpti}[1]{\mpt_{#1}}
\newcommand{\mptii}{\mpti{i}}
\newcommand{\mpii}{\mpi{i}}
\newcommand{\mps}{Q}
\newcommand{\mpsi}[1]{\mps_{#1}}
\newcommand{\mpsii}{\mpsi{i}}
\newcommand{\tmpt}{\tmp^T}
\newcommand{\mz}{Z}
\newcommand{\mv}{V}
\newcommand{\mvi}[1]{\mv_{#1}}
\newcommand{\mvt}{V^T}
\newcommand{\mvti}[1]{\mvt_{#1}}
\newcommand{\mzt}{\mz^T}
\newcommand{\tmz}{\tilde{\mz}}
\newcommand{\tmzt}{\tmz^T}
\newcommand{\mx}{\mathbf{X}}
\newcommand{\ma}{\mathbf{A}}
\newcommand{\Ft}{\mathbf{F}_{t}}
\newcommand{\invFt}{\mathbf{F}_{t}^{-1}}
\newcommand{\FtT}{\mathbf{F}_{t}^{\top}}
\newcommand{\invFtT}{(\FtT)^{-1}}
\newcommand{\mxs}[1]{\mx_{#1}}

\newcommand{\mai}[1]{\ma_{#1}}
\newcommand{\mat}{\tran{\ma}}
\newcommand{\mati}[1]{\mat_{#1}}

\newcommand{\mc}{{C}}
\newcommand{\mci}[1]{\mc_{#1}}
\newcommand{\mcti}[1]{\mct_{#1}}

\newcommand{\md}{{\mathbf{D}}}
\newcommand{\mdi}[1]{\md_{#1}}

\newcommand{\mxi}[1]{\textrm{diag}^2\paren{\vxi{#1}}}
\newcommand{\mxii}{\mxi{i}}

\newcommand{\hmx}{\hat{\mx}}
\newcommand{\hmxi}[1]{\hmx_{#1}}
\newcommand{\hmxii}{\hmxi{i}}
\newcommand{\hmxt}{\hmx^T}
\newcommand{\mxt}{\mx^\top}
\newcommand{\mi}{\mathbf{I}}
\newcommand{\mq}{Q}
\newcommand{\mqt}{\mq^T}
\newcommand{\mlam}{\Lambda}

\renewcommand{\L}{\mcal{L}}
\newcommand{\R}{\mcal{R}}
\newcommand{\X}{\mcal{X}}
\newcommand{\Y}{\mcal{Y}}
\newcommand{\F}{\mcal{F}}
\newcommand{\nur}[1]{\nu_{#1}}
\newcommand{\lambdar}[1]{\lambda_{#1}}
\newcommand{\gammai}[1]{\gamma_{#1}}
\newcommand{\gammaii}{\gammai{i}}
\newcommand{\alphai}[1]{\alpha_{#1}}
\newcommand{\alphaii}{\alphai{i}}
\newcommand{\lossp}[1]{\ell_{#1}}
\newcommand{\eps}{\epsilon}
\newcommand{\epss}{\eps^*}
\newcommand{\lsep}{\lossp{\eps}}
\newcommand{\lseps}{\lossp{\epss}}
\newcommand{\T}{\mcal{T}}

\newcommand{\kc}[1]{\begin{center}\fbox{\parbox{3in}{{\textcolor{green}{KC: #1}}}}\end{center}}
\newcommand{\edward}[1]{\begin{center}\fbox{\parbox{3in}{{\textcolor{red}{EM: #1}}}}\end{center}}
\newcommand{\nv}[1]{\begin{center}\fbox{\parbox{3in}{{\textcolor{blue}{NV: #1}}}}\end{center}}

\newcommand{\newstuffa}[2]{#2}
\newcommand{\newstufffroma}[1]{}
\newcommand{\newstufftoa}{}

\newcommand{\newstuff}[2]{#2}
\newcommand{\newstufffrom}[1]{}
\newcommand{\newstuffto}{}
\newcommand{\oldnote}[2]{}

\newcommand{\commentout}[1]{}
\newcommand{\mypar}[1]{\medskip\noindent{\bf #1}}

\newcommand{\inner}[2]{\left< {#1} , {#2} \right>}
\newcommand{\kernel}[2]{K\left({#1},{#2} \right)}
\newcommand{\tprr}{\tilde{p}_{rr}}
\newcommand{\hxr}{\hat{x}_{r}}
\newcommand{\projalg}{{PST }}
\newcommand{\projealg}[1]{$\textrm{PST}_{#1}~$}
\newcommand{\gradalg}{{GST }}

\newcounter {mySubCounter}
\newcommand {\twocoleqn}[4]{
  \setcounter {mySubCounter}{0} %
  \let\OldTheEquation \theequation %
  \renewcommand {\theequation }{\OldTheEquation \alph {mySubCounter}}%
  \noindent \hfill%
  \begin{minipage}{.40\textwidth}
\vspace{-0.6cm}
    \begin{equation}\refstepcounter{mySubCounter}
      #1
    \end {equation}
  \end {minipage}
~~~~~~
  \addtocounter {equation}{ -1}%
  \begin{minipage}{.40\textwidth}
\vspace{-0.6cm}
    \begin{equation}\refstepcounter{mySubCounter}
      #3
    \end{equation}
  \end{minipage}%
  \let\theequation\OldTheEquation
}

\newcommand{\vzero}{\mb{0}}

\newcommand{\smargin}{\mcal{M}}

\newcommand{\ai}[1]{A_{#1}}
\newcommand{\bi}[1]{B_{#1}}
\newcommand{\aii}{\ai{i}}
\newcommand{\bii}{\bi{i}}
\newcommand{\betai}[1]{\beta_{#1}}
\newcommand{\betaii}{\betai{i}}
\newcommand{\mar}{M}
\newcommand{\mari}[1]{\mar_{#1}}
\newcommand{\marii}{\mari{i}}
\newcommand{\nmari}[1]{m_{#1}}
\newcommand{\nmarii}{\nmari{i}}

\newcommand{\erf}{\Phi}

\newcommand{\var}{V}
\newcommand{\vari}[1]{\var_{#1}}
\newcommand{\varii}{\vari{i}}

\newcommand{\varb}{v}
\newcommand{\varbi}[1]{\varb_{#1}}
\newcommand{\varbii}{\varbi{i}}

\newcommand{\vara}{u}
\newcommand{\varai}[1]{\vara_{#1}}
\newcommand{\varaii}{\varai{i}}

\newcommand{\marb}{m}
\newcommand{\marbi}[1]{\marb_{#1}}
\newcommand{\marbii}{\marbi{i}}

\newcommand{\algname}{{AROW}}
\newcommand{\rlsname}{{RLS}}
\newcommand{\mrlsname}{{MRLS}}

\newcommand{\phia}{\psi}
\newcommand{\phib}{\xi}

\newcommand{\amsigmaii}{\tilde{\msigma}_t}
\newcommand{\amsigmai}[1]{\tilde{\msigma}_{#1}}
\newcommand{\avmuii}{\tilde{\vmu}_i}
\newcommand{\avmui}[1]{\tilde{\vmu}_{#1}}
\newcommand{\amarbii}{\tilde{\marb}_i}
\newcommand{\avarbii}{\tilde{\varb}_i}
\newcommand{\avaraii}{\tilde{\vara}_i}
\newcommand{\aalphaii}{\tilde{\alpha}_i}

\newcommand{\svar}{v}
\newcommand{\smar}{m}
\newcommand{\nsmar}{\bar{m}}

\newcommand{\vnu}{\mb{\nu}}
\newcommand{\vnut}{\vnu^\top}
\newcommand{\vz}{\mb{z}}
\newcommand{\vZ}{\mb{Z}}
\newcommand{\fphi}{f_{\phi}}
\newcommand{\gphi}{g_{\phi}}


\newcommand{\vtmui}[1]{\tilde{\vmu}_{#1}}
\newcommand{\vtmuii}{\vtmui{i}}

\newcommand{\zetai}[1]{\zeta_{#1}}
\newcommand{\zetaii}{\zetai{i}}


\newcommand{\vstate}{\bf{s}}
\newcommand{\vstatet}[1]{\vstate_{#1}}
\newcommand{\vstatett}{\vstatet{t}}

\newcommand{\mtran}{\bf{\Phi}}
\newcommand{\mtrant}[1]{\mtran_{#1}}
\newcommand{\mtrantt}{\mtrant{t}}

\newcommand{\vstatenoise}{\bf{\eta}}
\newcommand{\vstatenoiset}[1]{\vstatenoise_{#1}}
\newcommand{\vstatenoisett}{\vstatenoiset{t}}

\newcommand{\vobser}{\bf{o}}
\newcommand{\vobsert}[1]{\vobser_{#1}}
\newcommand{\vobsertt}{\vobsert{t}}

\newcommand{\mobser}{\bf{H}}
\newcommand{\mobsert}[1]{\mobser_{#1}}
\newcommand{\mobsertt}{\mobsert{t}}

\newcommand{\vobsernoise}{\bf{\nu}}
\newcommand{\vobsernoiset}[1]{\vobsernoise_{#1}}
\newcommand{\vobsernoisett}{\vobsernoiset{t}}

\newcommand{\mstatenoisecov}{\bf{Q}}
\newcommand{\mstatenoisecovt}[1]{\mstatenoisecov_{#1}}
\newcommand{\mstatenoisecovtt}{\mstatenoisecovt{t}}

\newcommand{\mobsernoisecov}{\bf{R}}
\newcommand{\mobsernoisecovt}[1]{\mobsernoisecov_{#1}}
\newcommand{\mobsernoisecovtt}{\mobsernoisecovt{t}}

\newcommand{\vestate}{\bf{\hat{s}}}
\newcommand{\vestatet}[1]{\vestate_{#1}}
\newcommand{\vestatett}{\vestatet{t}}
\newcommand{\vestatept}[1]{\vestatet{#1}^+}
\newcommand{\vestatent}[1]{\vestatet{#1}^-}

\newcommand{\mcovar}{\bf{P}}
\newcommand{\mcovart}[1]{\mcovar_{#1}}
\newcommand{\mcovarpt}[1]{\mcovart{#1}^+}
\newcommand{\mcovarnt}[1]{\mcovart{#1}^-}

\newcommand{\mkalmangain}{\bf{K}}
\newcommand{\mkalmangaint}[1]{\mkalmangain_{#1}}

\newcommand{\vkalmangain}{\bf{\kappa}}
\newcommand{\vkalmangaint}[1]{\vkalmangain_{#1}}

\newcommand{\obsernoise}{{\nu}}
\newcommand{\obsernoiset}[1]{\obsernoise_{#1}}
\newcommand{\obsernoisett}{\obsernoiset{t}}

\newcommand{\obsernoisecov}{r}
\newcommand{\obsernoisecovt}[1]{\obsernoisecov_{#1}}
\newcommand{\obsernoisecovtt}{\obsernoisecov}

\newcommand{\obsnscv}{s}
\newcommand{\obsnscvt}[1]{\obsnscv_{#1}}
\newcommand{\obsnscvtt}{\obsnscvt{t}}

\newcommand{\Psit}[1]{\Psi_{#1}}
\newcommand{\Psitt}{\Psit{t}}

\newcommand{\Omegat}[1]{\Omega_{#1}}
\newcommand{\Omegatt}{\Omegat{t}}

\newcommand{\ellt}[1]{\ell_{#1}}
\newcommand{\gllt}[1]{g_{#1}}

\newcommand{\chit}[1]{\chi_{#1}}

\newcommand{\ms}{\mathcal{M}}
\newcommand{\us}{\mathcal{U}}
\newcommand{\as}{\mathcal{A}}

\newcommand{\mn}{M}
\newcommand{\un}{U}

\newcommand{\seti}[1]{S_{#1}}

\newcommand{\obj}{\mcal{C}}

\newcommand{\dta}[3]{d_{#3}\paren{#1,#2}}

\newcommand{\coa}{a}
\newcommand{\coc}{c}
\newcommand{\cob}{b}
\newcommand{\cor}{r}
\newcommand{\conu}{\nu}

\newcommand{\coat}[1]{\coa_{#1}}
\newcommand{\coct}[1]{\coc_{#1}}
\newcommand{\cobt}[1]{\cob_{#1}}
\newcommand{\cort}[1]{\cor_{#1}}
\newcommand{\conut}[1]{\conu_{#1}}

\newcommand{\coatt}{\coat{t}}
\newcommand{\coctt}{\coct{t}}
\newcommand{\cobtt}{\cobt{t}}
\newcommand{\cortt}{\cort{t}}
\newcommand{\conutt}{\conut{t}}

\newcommand{\rb}{R_B}
\newcommand{\proj}{\textrm{proj}}

\section{Introduction}
 We consider the online binary classification task, in which a learning algorithm predicts a binary label given inputs in a sequence of rounds. An example of such task is classification of emails based on their content as spam or not spam. Traditionally, the purpose of a learning algorithm is to make the number of mistakes as small as possible compared to predictions of some \emph{single} function from some class. We call this setting the \emph{stationary} setting.

Following the pioneering work of Rosenblatt \cite{Rosenblatt58} many algorithms were proposed for this setting. Some of them are able to employ second-order information. For example, the second-order perceptron algorithm \cite{CesaBianchiCoGe05} extends the original perceptron algorithm and uses the spectral properties of the data to improve performance. Another example is the AROW algorithm \cite{CrammerKuDr09} which uses confidence as a second-order information. All these second-order algorithms can be seen as RLS (Regularized Least Squares) based, as their update equations are similar to those of RLS, updating a weight vector and a covariance-like matrix. Under the stationary setting, RLS-based second-order algorithms have been successfully applied to the regression and classification tasks, as shown in \tabref{table:rls_algorithms}.

Despite the extensive and impressive guarantees that can be made for algorithms
in such setting \cite{CesaBianchiCoGe05,CrammerKuDr09}, competing with the best \emph{fixed} function is not always good
enough. In many real-world applications, the true target function is not fixed, but is
slowly changing over time, or switching from time to time. These reasons led to the development of algorithms and accompanying analysis
for drifting and shifting settings, which we collectively call the \emph{non-stationary} setting. For online regression, few algorithms were developed for this setting \cite{HerbsterW01,VaitsCr11,MoroshkoCr13}. Yet, for online classification, the Shifting Perceptron algorithm \cite{CavallantiCG07} is a first-order algorithm that shrinks the weight vector each iteration, and in this way weaken dependence on the past. The Modified Perceptron algorithm \cite{DBLP:journals/algorithmica/BlumFKV98} is another first-order algorithm that had been shown to work well in the drifting setting \cite{DBLP:conf/colt/CrammerMEV10}. In this paper we derive a new RLS-based \emph{second-order} algorithm for classification, designed to work with target drift, and thus we fill the missing configuration in \tabref{table:rls_algorithms}. Our algorithm extends the second-order perceptron algorithm \cite{CesaBianchiCoGe05}, and we provide a performance bound in the mistake bound model.

A practical variant of the fully supervised online classification setting, 
is where, at each prediction step, the learner can abstain from observing the current label. This setting is called selective sampling \cite{FreundSeShTi97}. In this setting a learning algorithm actively decides when to query for a label. If the label is queried, then the label value can be used to improve future predictions, and otherwise the algorithm never knows whether his prediction was correct. Roughly speaking, selective sampling algorithms can be divided in two groups. In the first group, a simple randomized rule is used to turn fully supervised algorithm to selective sampling algorithm. The rule uses the margin of the estimate. This group includes the selective sampling versions of the perceptron and the second-order perceptron algorithms \cite{Cesa-BianchiGZ06a}. In the second group, selective sampling algorithms are derived based on comparing the variance of the RLS estimate to some threshold. This group includes the BBQ algorithm \cite{DBLP:conf/icml/Cesa-BianchiGO09}, where the threshold decays polynomially with $t$ as $t^{-\kappa}$, and more involved variants where the threshold depends on the margin of the RLS estimate \cite{DBLP:conf/colt/DekelGS10, DBLP:conf/icml/OrabonaC11}.

In all previous work on selective sampling the performance of an algorithm is compared to the performance of a {\em single} linear comparator. To the best of our knowledge, our work is the first instance of learning online in the context of drifting in the selective sampling setting. We build on the work of Cesa-Bianchi et al~\cite{Cesa-BianchiGZ06a} that combined a randomized rule into the Perceptron algorithm, yielding a selective sampling algorithm. We analyze the resulting algorithm in the drifting setting, and derive a bound on the expected number of mistakes of the algorithm. Thus, we fill the non-stationary cell in \tabref{table:rand_ss_algorithms}. Simulations on synthetic and real-world datasets show the advantages of our algorithm, in a fully supervised and selective sampling settings.
\begin{table}[t]
\center
\begin{tabulary}{1\textwidth}{|C|C|C|}
\hline
   & \textbf{Stationary} & \textbf{Non-stationary} \\
\hline
  \textbf{Regression} & \cite{Vovk01,AzouryWa01,Forster}  & \cite{MoroshkoCr13,VaitsCr11} \\
\hline
  \textbf{Classification} & \cite{CesaBianchiCoGe05,CrammerKuDr09} & This work \\
\hline
\end{tabulary}
\caption{Fully supervised online RLS-based second-order algorithms.}
\vspace{-0.4cm}
\label{table:rls_algorithms}
\end{table}

\vspace{-0.25cm}
\section{Problem setting}
\vspace{-0.25cm}
We consider the standard online learning model \cite{DBLP:journals/ml/Angluin87, DBLP:journals/ml/Littlestone87} for binary classification, in which learning proceeds in a sequence of rounds $t=1,2,\ldots ,T$. In 
round $t$ the
algorithm observes an instance $x_t\in\reals^d$ and outputs a prediction $\hat{y}_{t}\in\{-1,+1\}$ for the label $y_t$ associated with
$x_t$. We say that the algorithm has made a prediction mistake if $\hat{y}_{t} \neq y_{t}$, and denote by $M_t$ the indicator function of the event $\hat{y}_{t} \neq y_{t}$. After observing the correct label $y_t$ the algorithm may update its prediction rule, and then proceeds to the next round. We denote by $m$ the total number of mistakes over a sequence of $T$ examples.

The performance of an algorithm is measured by the total number
of mistakes it makes on an arbitrary sequence of examples. In the standard
performance model, the goal is to bound this total number of mistakes in terms of
the performance of the best \emph{fixed} linear classifier $u\in\reals^d$ in hindsight. Since finding $u\in\reals^d$ that minimizes the number of mistakes on a known sequence is a computationally hard problem, the performance of the best predictor in hindsight is often measured using the cumulative hinge loss $L_{\gamma,T}\left(u\right)=\sum_{t=1}^{T}\ell_{\gamma,t}\left(u\right)$, where $\ell_{\gamma,t}\left(u\right)=\max\left\{ 0,\gamma-y_{t}u^{\top}x_{t}\right\}$ is the hinge loss of the competitor $u$ on round $t$ for some margin threshold $\gamma>0$.

In the drifting setting that we consider in this work, the learning algorithm faces the harder goal of bounding its total number of mistakes in terms of the cumulative hinge loss achieved by an arbitrary sequence $u_1,u_2,\ldots,u_T \in\reals^d $ of comparison vectors. The cumulative hinge loss of such sequence is $L_{\gamma,T}\left( \{u_t\}\right)=\sum_{t=1}^{T}\ell_{\gamma,t}\left(u_t\right)$.
 To make this goal feasible, the bound is allowed to scale also with the norm of $u_1$ and the total amount of drift defined to be
$V = V(\{u_t\}) = \sum_{t=2}^{T}\left\Vert u_{t}-u_{t-1}\right\Vert ^{2}$.

We consider two settings: (a) standard supervised online binary
classification (described above), and (b) selective sampling. In the
later setting, after each prediction the learner may observe the
correct label $y_t$ only by issuing a query. If no query is issued at
time $t$, then $y_t$ remains unknown. We represent the algorithm's
decision of querying the label at time $t$ through the value of a
Bernoulli random variable $Z_t$, and the event of a mistake with the
indicator variable $M_t=1$. Note that we measure the performance of
the algorithm by the total number of mistakes it makes on a sequence
of examples, including the rounds where the true label $y_t$ remains
unknown.

Finally, $\bar{L}_{\gamma,T}\left(\{u_{t}\}\right)=\mathbb{E}\left[\sum_{t=1}^{T}M_{t}Z_{t}\ell_{\gamma,t}\left(u_{t}\right)\right]$
 is the expected total hinge loss of a competitor on mistaken and queried rounds, and trivially $\bar{L}_{\gamma,T}\left(\{u_{t}\}\right) \leq L_{\gamma,T}\left(\{u_t\}\right)$.
\begin{table}[t]
\center
\begin{tabulary}{1\textwidth}{|C|C|}
\hline
   \textbf{Stationary} & \textbf{Non-stationary} \\
\hline
   \cite{Cesa-BianchiGZ06a} & This work \\
\hline
\end{tabulary}
\caption{Second-order randomized {\em selective sampling} algorithms for classification.}
\vspace{-0.3cm}
\label{table:rand_ss_algorithms}
\end{table}

\vspace{-0.25cm}
\section{Algorithms}
\vspace{-0.25cm}
Online algorithms work in rounds. On round $t$ the algorithm receives
an input $x_t$ and makes a prediction. We follow Moroshko and
Crammer~\cite{MoroshkoCr13} and design the prediction as a
last-step min-max problem in the context of drifting. Yet, unlike all
previous work, we design algorithms for
classification. Specifically, prediction is the solution of the
following optimization problem,
\begin{align}
\hat{y}_{T}=\arg\min_{\hat{y}_{T}\in\{-1,+1\}}\max_{y_{T}\in\{-1,+1\} }\Bigg[\sum_{t=1}^{T}\left(y_{t}-\hat{y}_{t}\right)^{2} \nonumber
\\-\min_{u_{1},\ldots,u_{T}}Q_{T}\left(u_{1},\ldots,u_{T}\right)\Bigg]~,\label{minmax_1}
\end{align}
where
\begin{align*}
Q_{t}\left(u_{1},\ldots,u_{t}\right) = & b\left\Vert u_{1}\right\Vert ^{2}+c\sum_{s=1}^{t-1}\left\Vert u_{s+1}-u_{s}\right\Vert ^{2}\\
& +\sum_{s=1}^{t}\left(y_{s}-u_{s}^{\top}x_{s}\right)^{2} 
\end{align*}
 for some positive constants $b,c$\footnote{We still use the squared
  loss in \eqref{minmax_1}, as done for least-squares
  SVMs~\cite{DBLP:journals/npl/SuykensV99,suykens2002least}, which allows us to compute all quantities analytically.}.

This optimization problem can also be seen as a game where the
algorithm chooses a prediction label $\hat{y}_t \in\{-1,+1\}$ to minimize the last-step
regret, while an adversary chooses a target label ${y}_t \in\{-1,+1\}$ to
maximize it.
The first term of \eqref{minmax_1}
is the loss suffered by the algorithm while
$Q_{t}\left(u_{1},\ldots,u_{t}\right)$ 
 is a sum of the loss suffered by some sequence of linear
functions $\{u_s\}_{s=1}^t$, a
penalty for consecutive pairs that are far from each other, and for the
norm of the first to be far from zero.

The following lemma enables to solve \eqref{minmax_1} by specifying
means to solve the inner optimization problem in \eqref{minmax_1}.
\begin{lemma}[{\cite{MoroshkoCr13}, Lemma 2}]
\label{lem:lemma12}
Denote \\
$P_{t}\left(u_{t}\right)=\min_{u_{1},\ldots,u_{t-1}}Q_{t}\left(u_{1},\ldots,u_{t}\right)$.
Then\\
 $P_{t}\left(u_{t}\right)=u_{t}^{\top}D_{t}u_{t}-2u_{t}^{\top}e_{t}+f_{t}$
where,
\begin{align}
&D_{1} \!=bI+x_{1}x_{1}^{\top}
&&
D_{t}=\left(D_{t-1}^{-1}+c^{-1}I\right)^{-1}+x_{t}x_{t}^{\top} \label{D}\\
&e_{1}\!=y_{1}x_{1} &&
e_{t}=\left(\!I\!+\!c^{-1}D_{t-1}\!\right)^{-1}\!e_{t-1}\!+\!y_{t}x_{t} \label{e}\\
&f_{1}\!=y_{1}^{2}&&
f_{t}\!=\!f_{t-1}-e_{t-1}^{\top}\left(\!cI\!+\!D_{t-1}\!\right)^{-1}\!\!e_{t-1}\!+\!y_{t}^{2}\nonumber
,
\end{align}
where, $D_{t}\in\mathbb{R}^{d\times d}$ is a PSD matrix, 
$e_{t}\in\mathbb{R}^{d\times1}$ and $f_{t}\in\mathbb{R}$.
\end{lemma}
From the lemma we solve,
$\min_{u_{1},\ldots,u_{t}}Q_{t}\left(u_{1},\ldots,u_{t}\right)$, by,
\begin{equation}
\min_{u_{1},\ldots,u_{t}}Q_{t}\left(u_{1},\ldots,u_{t}\right)
= \min_{u_{t}} P_t(u_{t})
= -e_{t}^{\top}D_{t}^{-1}e_{t}+f_{t} \label{bb}~.
\end{equation}
Next, we substitute the value of $e_{T}$ from \eqref{e} (as a
function of $\yi{T}$) in \eqref{bb}, and then substitute \eqref{bb} in \eqref{minmax_1}. Omitting terms not depending explicitly on $\yi{T}$ and $\hat{y}_T$ we get from \eqref{minmax_1} that,
\begin{align*}
\hat{y}_{T}&=\arg\min_{\hat{y}_{T}\in\{-1,+1\}}\max_{y_{T}\in\{-1,+1\} }\Bigg[\left({x}_{T}^{\top}{D}_{T}^{-1}{x}_{T}\right)y_{T}^{2}\\
&
\!\!\! +2y_{T}\left({x}_{T}^{\top}{D}_{T}^{-1}\left({I}+c^{-1}{D}_{T-1}\right)^{-1}{e}_{T-1}-\hat{y}_{T}\right)+\hat{y}_{T}^{2}\Bigg]\\
&=\arg\min_{\hat{y}_{T}\in\{-1,+1\}}\Bigg[{x}_{T}^{\top}{D}_{T}^{-1}{x}_{T} \\
&+2 \left\vert {x}_{T}^{\top}{D}_{T}^{-1}\left({I}+c^{-1}{D}_{T-1}\right)^{-1}{e}_{T-1}-\hat{y}_{T}\right\vert
+\hat{y}_{T}^{2}\Bigg]\\
&= \sign(\hat{p}_t)|_{t=T}~,
\end{align*}
where
\begin{align}
\hat{p}_t=x_{t}^{\top}D_{t}^{-1}\left(I+c^{-1}D_{t-1}\right)^{-1}e_{t-1}~. \label{my_predictor}
\end{align}
To the best of our knowledge, this is the first application of the last-step min-max
approach directly for classification, and not as a reduction from
regression, which is possible 
 by employing the square loss, as in least-squares
SVMs~\cite{DBLP:journals/npl/SuykensV99,suykens2002least}.
Indeed, we showed that the optimal prediction for
classification is the sign of the optimal prediction for regression \cite{MoroshkoCr13}.

Our algorithm includes the second-order perceptron~\cite{CesaBianchiCoGe05} algorithm as a
special case when $c=\infty$.
The second-order perceptron algorithm is indeed using the sign
of the optimal min-max prediction for regression \cite{Forster}, which is in fact the prediction
of the AAR algorithm \cite{Vovk01} (aka "forward algorithm"~\cite{AzouryWa01}).
Additionally, similar to other algorithms
\cite{Rosenblatt58,CesaBianchiCoGe05}, we update the algorithm only on
mistaken rounds. We call the algorithm LASEC for last-step adaptive
classifier. LASEC is a special case of \figref{algorithm:ss_lasec} when
setting $a=\infty$ (see below). Note that in the pseudocode
two indices are used, the current time $t$ and the number of examples
used to update the model $k$. This makes the presentation simpler as some examples are not used to update the
model, the ones for which there was no classification mistake.
Note, the update equations of \figref{algorithm:ss_lasec} are essentially \eqref{D} and \eqref{e}. The LASEC algorithm can be seen as an extension to the non-stationary setting of the second-order perceptron algorithm \cite{CesaBianchiCoGe05}. Indeed, for $c=\infty$ the LASEC algorithm is reduced to the second-order perceptron algorithm.

Next, we turn LASEC from an algorithm that uses the labels of all
inputs to one that queries labels stochastically. Specifically, the
algorithm uses the margin $\vert \hat{p}_t \vert$ defined in
\eqref{my_predictor} to randomly choose whether to make a
prediction. We interpret large values of the margin $\vert \hat{p}_t \vert$
as being confident in the prediction, which should reduce the
probability of querying the label. Specifically, the algorithm is
querying a label with probability $a/(a+|\hat{p}_{t}|)$ for some
$a>0$. If $a\rightarrow\infty$ the algorithm will always query, and
reduce to LASEC, while if $a\rightarrow 0$ it will never query.

This approach for deriving selective-algorithms from margin-based
online algorithms is not new, and was used to design an algorithm for
the non-drifting case~
\cite{Cesa-BianchiGZ06a}. Yet, unlike other selective sampling
algorithms
\cite{Cesa-BianchiGZ06a,DBLP:conf/icml/Cesa-BianchiGO09,DBLP:conf/colt/DekelGS10,DBLP:conf/icml/OrabonaC11},
our algorithm is designed to work in the {\em drifting setting}. Since
the algorithm is based on the LASEC algorithm, we call it LASEC-SS,
where SS stands for selective sampling. The algorithm is summarized in
\figref{algorithm:ss_lasec} as well.  Note that LASEC-SS includes other
algorithms as special cases. Specifically,  LASEC-SS
 is reduced for $c=\infty$ to the selective sampling version of the
second-order perceptron algorithm~\cite{Cesa-BianchiGZ06a}, and as
mentioned above, for $a=\infty$ it is reduced to LASEC, and the setting
of both $c=\infty,a=\infty$ reduces the algorithm to the second-order
perceptron, which in turn reduces to the perceptron algorithm for
$b\rightarrow\infty$.

The algorithm is flexible enough to be
tuned both to drifting or non-drifting setting (using $c$),
between selective sampling or supervised learning (using $a$), and
between first-order or second-order modeling (using $b$).

Our algorithm can be combined with Mercer kernels as it employs only
sums of inner- and outer-products of the inputs. This allows it to
build non-linear models (e.g.~\cite{ScholkopfSm02}).

\begin{figure}[t!]
{
\paragraph{Parameters:} $0<b<c~,~0<a$
\paragraph{Initialize:} Set
$D_0=(bc)/(c-b)\,I\in\reals^{d\times d}$ , $e_0=0\in\reals^d$ and $k=1$\\
{\bf For $t=1 \comdots T$} do
\begin{itemize}
\nolineskips
\item Receive an instance $x_{t}\in\reals^d$
\item Set
\begin{align*}
  &  S_t=\left(D_{k-1}^{-1}+c^{-1}I\right)^{-1}+x_{t}x_{t}^{\top}&
  \textrm{see }
  \eqref{D}\\
  &
  \hat{p}_{t}=x_{t}^{\top}S_t^{-1}\left(I+c^{-1}D_{k-1}\right)^{-1}e_{k-1}&\textrm{see }\eqref{my_predictor}
\end{align*}
\item Output  prediction $\hyi{t}=\sign(\hat{p}_{t})$
\item Draw a Bernoulli random variable $Z_t\in \{0,1\}$ of parameter $\frac{a}{a+|\hat{p}_{t}|}$
\item If $Z_t=1$ then query label $y_t\in \{-1,+1\}$ and if $\hyi{t}\neq y_t$ then update:
\begin{align*}
&e_k=\left(I+c^{-1}D_{k-1}\right)^{-1}e_{k-1}+y_{t}x_{t}&\textrm{see }\eqref{e}\\
&D_k=S_t&\\
&k\leftarrow k+1&
\end{align*}
\end{itemize}
}
\figline
\caption{LASEC for selective sampling. Set $a=\infty$ for the
  supervised setting.}
\label{algorithm:ss_lasec}
\end{figure}

\section{Analysis}
We now prove bounds for the number of mistakes of our algorithm. We provide a mistake bound for the fully supervised version (LASEC) and for the selective sampling version (LASEC-SS).
Our bounds depend on the total drift of the reference sequence $V_m=\sum_{k=2}^{m}\left\Vert u_{k}-u_{k-1}\right\Vert ^{2}$ which is calculated on rounds when the algorithm makes updates.
We denote by $\mathcal{M}\subseteq \{1,2,\ldots\}$ the set of indices when the algorithm updates. For the supervised setting it is the set of mistaken trials ($M_t=1$). For the selective sampling setting it is the set of indices when $Z_t=1$ and $M_t=1$.

\begin{theorem}
\label{thm:lasec_bound}
Assume the LASEC algorithm (of \figref{algorithm:ss_lasec} with $a=\infty$) is run on a finite sequence of examples.
 Then for any reference sequence $\{u_t\}$ and $\gamma>0$ the number $m=\vert \mathcal{M}\vert$ of mistakes satisfies
\begin{align}
&m\leq \frac{1}{\gamma}L_{\gamma,T}\left( \{u_t\}\right)\nonumber \\
&+\frac{1}{\gamma}\sqrt{\left(b\left\Vert u_{1}\right\Vert ^{2}+cV_m+\sum_{k=1}^{m}\left(u_{k}^{\top}x_{k}\right)^{2}\right)\sum_{t\in\mathcal{M}}x_{t}^{\top}D_{k}^{-1}x_{t}}
\label{laser_bound}
\end{align}
\end{theorem}
\begin{remark}
For the stationary case, when $u_{k}=u$ $\forall k$ ($V_m=0$) and
we set $c=\infty$ for the LASEC algorithm we recover the second-order
perceptron bound \cite{CesaBianchiCoGe05}.
\end{remark}
\begin{theorem}
\label{thm:lasec_ss_bound}
Assume the LASEC-SS algorithm of \figref{algorithm:ss_lasec} is run on a sequence of $T$ examples with parameter $a>0$.
 Then for any reference sequence $\{u_t\}$ and $\gamma>0$ the expected
 number of mistakes satisfies
\begin{align}
&\mathbb{E}\left[\sum_{t=1}^{T}M_{t}\right]  \leq  \frac{1}{\gamma}\bar{L}_{\gamma,T}\left(\{u_{t}\}\right) \nonumber \\
&+ \frac{a}{2\gamma^{2}}\left(b\left\Vert u_{1}\right\Vert ^{2}+cV_m +\mathbb{E}\left[\sum_{t=1}^{T}M_{t}Z_{t}\left(u_{t}^{\top}x_{t}\right)^{2}\right]\right) \nonumber \\
& +\frac{1}{2a}\mathbb{E}\left[\sum_{t=1}^{T}M_{t}Z_{t}x_{t}^{\top}D_{t}^{-1}x_{t}\right]~.
\label{a3}
\end{align}
Moreover, the expected number of labels queried by the algorithm equals $\sum_{t=1}^{T}\mathbb{E}\left[\frac{a}{a+|\hat{p}_{t}|}\right]$~.
\end{theorem}
\begin{remark}
As in other context~\cite{Cesa-BianchiGZ06a}: \thmref{thm:lasec_bound} is not a special case of \thmref{thm:lasec_ss_bound}. Indeed, setting $a=\infty$  makes the bound of \thmref{thm:lasec_ss_bound} unbounded, as opposed to \thmref{thm:lasec_bound}.  Also, from the last part of \thmref{thm:lasec_ss_bound} we observe that more labels would be queried for larger values of $a$. However, the tradeoff between number of queries and mistakes is not clear.
\end{remark}
\begin{remark}
For the stationary case, when $u_{k}=u$ $\forall k$ ($V_m=0$) and
we set $c=\infty$ for the LASEC-SS algorithm we recover the bound
of the selective sampling version of the second-order perceptron algorithm
\cite{Cesa-BianchiGZ06a}.
\end{remark}
The bound \eqref{a3} depends on the 
parameter $a$. If we would know the future, by setting,
\[
a=\gamma\sqrt{\frac{\mathbb{E}\left[\sum_{t=1}^{T}M_{t}Z_{t}x_{t}^{\top}D_{t}^{-1}x_{t}\right]}{b\left\Vert u_{1}\right\Vert ^{2}+cV_m+\mathbb{E}\left[\sum_{t=1}^{T}M_{t}Z_{t}\left(u_{t}^{\top}x_{t}\right)^{2}\right]}}
\]
 we would minimize the bound and get
\begin{align*}
&\mathbb{E}\left[\sum_{t=1}^{T}M_{t}\right] \leq  \frac{1}{\gamma}\bar{L}_{\gamma,T}\left(\{u_{t}\}\right)
  \\
& + \frac{1}{\gamma}\sqrt{
\begin{aligned}
\left(b\left\Vert u_{1}\right\Vert ^{2}+cV_m+\mathbb{E}\left[\sum_{t=1}^{T}M_{t}Z_{t}\left(u_{t}^{\top}x_{t}\right)^{2}\right]\right) \\
\times\left(\mathbb{E}\left[\sum_{t=1}^{T}M_{t}Z_{t}x_{t}^{\top}D_{t}^{-1}x_{t}\right]\right)
\end{aligned}
}
~.
\end{align*}

The last bound is an expectation version of the mistake bound for the (deterministic) LASEC
algorithm of \thmref{thm:lasec_bound}, and it might be even sharper than the LASEC bound, since the magnitude of the three quantities
$\bar{L}_{\gamma,T}\left(\{u_{t}\}\right)$,
$\mathbb{E}\left[\sum_{t=1}^{T}M_{t}Z_{t}\left(u_{t}^{\top}x_{t}\right)^{2}\right]$ and
$\mathbb{E}\left[\sum_{t=1}^{T}M_{t}Z_{t}x_{t}^{\top}D_{t}^{-1}x_{t}\right]$
is ruled by the size of the random set of updates
$\{t: Z_tM_t=1\}$, which is typically smaller than the set of mistaken trials of the deterministic algorithm.

We now prove the bounds in \thmref{thm:lasec_bound} and \thmref{thm:lasec_ss_bound}, in the following unified proof.
\begin{proof}
Consider only the rounds $t$ when the algorithm makes an update,
that is $t\in\mathcal{M}$. 
Noting that our choice $\hat{p}_{t}=x_{t}^{\top}S_{t}^{-1}\left(I+c^{-1}D_{k-1}\right)^{-1}e_{k-1}$ (in \figref{algorithm:ss_lasec}) is the same as the prediction of the LASER algorithm for regression
with drift, we can use the result proven by Moroshko and Crammer \cite{MoroshkoCr13}
(Theorem 4 therein), from where we have that for any sequence $u_{1},\ldots,u_{m}$
\begin{align}
\sum_{t\in\mathcal{M}}\left(\hat{p}_{t}-y_{t}\right)^{2}  \leq & b\left\Vert u_{1}\right\Vert ^{2}+c\sum_{k=2}^{m}\left\Vert u_{k}-u_{k-1}\right\Vert ^{2} \nonumber \\
& \!\!\!\!\!\!\!\!\!\!\!\!\!\!\!\!\! +\sum_{k=1}^{m}\left(y_{k}-u_{k}^{\top}x_{k}\right)^{2}+\sum_{t\in\mathcal{M}}x_{t}^{\top}D_{k}^{-1}x_{t}~.
\label{a1}
\end{align}
Note that in \eqref{a1} the sums are over rounds when the algorithm
makes an update (for simplicity, we write $y_{k}$ as shorthand for $y_{t_{k}}$ where $t_k\in\mathcal{M}$).
Expanding the squares in \eqref{a1}, lower bound $\hat{p}_{t}^{2}\geq0$ and
substituting $y_{t}\hat{p}_{t}=-\left|\hat{p}_{t}\right|$ when $t\in\mathcal{M}$
we obtain
\begin{align*}
\sum_{t\in\mathcal{M}}\left|\hat{p}_{t}\right|  \leq & \frac{b}{2}\left\Vert u_{1}\right\Vert ^{2}+\frac{c}{2}V_m-\sum_{k=1}^{m}y_{k}u_{k}^{\top}x_{k}\\
&+\frac{1}{2}\sum_{k=1}^{m}\left(u_{k}^{\top}x_{k}\right)^{2}+\frac{1}{2}\sum_{t\in\mathcal{M}}x_{t}^{\top}D_{k}^{-1}x_{t}~.
\end{align*}
 The last bound is correct for any sequence $u_{k}$. We replace $u_{k}$
with $\frac{a}{\gamma}u_{k}$ (for some $a>0$) and get
\begin{align*}
\sum_{t\in\mathcal{M}}\left|\hat{p}_{t}\right|  \leq & b\frac{a^{2}}{2\gamma^{2}}\left\Vert u_{1}\right\Vert ^{2}+c\frac{a^{2}}{2\gamma^{2}}V_m-\frac{a}{\gamma}\sum_{k=1}^{m}y_{k}u_{k}^{\top}x_{k}\\
&+\frac{a^{2}}{2\gamma^{2}}\sum_{k=1}^{m}\left(u_{k}^{\top}x_{k}\right)^{2}+\frac{1}{2}\sum_{t\in\mathcal{M}}x_{t}^{\top}D_{k}^{-1}x_{t}~.
\end{align*}
Using $\gamma-\ell_{\gamma,t}\left(u_{t}\right)\leq
y_{t}u_{t}^{\top}x_{t}$, which follows the definition of the hinge loss,
we get
\begin{align}
\sum_{t\in\mathcal{M}}\left(\left|\hat{p}_{t}\right|+a\right) & \leq   \frac{a}{\gamma}\sum_{t\in\mathcal{M}}\ell_{\gamma,t}\left(u_{t}\right)
+\frac{1}{2}\sum_{t\in\mathcal{M}}x_{t}^{\top}D_{k}^{-1}x_{t}
\nonumber \\
& \!\!\!\!\!\!\!\!\!\!\!\!\!\!\!\!\!
+ \frac{a^{2}}{2\gamma^{2}}\left(b\left\Vert u_{1}\right\Vert ^{2}+cV_m +\sum_{k=1}^{m}\left(u_{k}^{\top}x_{k}\right)^{2}\right) ~.
\label{a2}
\end{align}
Next, to prove the bound for the LASEC algorithm in \thmref{thm:lasec_bound}, we further bound
$\left|\hat{p}_{t}\right|\geq0$ in \eqref{a2} and then divide it by $a$. We obtain
\begin{align*}
m  \leq & \frac{1}{\gamma}\sum_{t\in\mathcal{M}}\ell_{\gamma,t}\left(u_{t}\right) +\frac{1}{2a}\sum_{t\in\mathcal{M}}x_{t}^{\top}D_{k}^{-1}x_{t}\\
&+ \frac{a}{2\gamma^{2}}\left(b\left\Vert u_{1}\right\Vert ^{2}+cV_m +\sum_{k=1}^{m}\left(u_{k}^{\top}x_{k}\right)^{2}\right)
~.
\end{align*}
The last bound is minimized by setting
\[
a=\gamma\sqrt{\frac{\sum_{t\in\mathcal{M}}x_{t}^{\top}D_{k}^{-1}x_{t}}{b\left\Vert u_{1}\right\Vert ^{2}+cV_m+\sum_{k=1}^{m}\left(u_{k}^{\top}x_{k}\right)^{2}}}~,
\]
and by using $\sum_{t\in\mathcal{M}}\ell_{\gamma,t}\left(u_{t}\right) \leq\ L_{\gamma,T}\left( \{u_t\}\right)$ we get the desired bound of \thmref{thm:lasec_bound},
\begin{align*}
&m\leq  \frac{1}{\gamma}L_{\gamma,T}\left( \{u_t\}\right)\\
&+\frac{1}{\gamma}\sqrt{\left(b\left\Vert u_{1}\right\Vert ^{2}+cV_m+\sum_{k=1}^{m}\left(u_{k}^{\top}x_{k}\right)^{2}\right)\sum_{t\in\mathcal{M}}x_{t}^{\top}D_{k}^{-1}x_{t}}~.
\end{align*}

To prove the mistake bound for the LASEC-SS algorithm in \thmref{thm:lasec_ss_bound}, we note that the sum
$\sum_{t\in\mathcal{M}}\left(\left|\hat{p}_{t}\right|+a\right)$ on
the LHS of \eqref{a2} can be written as $\sum_{t}M_{t}Z_{t}\left(\left|\hat{p}_{t}\right|+a\right)$.
Taking expectation on both sides of \eqref{a2} and using $\mathbb{E}Z_{t}=a/\left(a+\left|\hat{p}_{t}\right|\right)$
we bound the expected number of mistakes of the algorithm,
\begin{align*}
&\mathbb{E}\left[\sum_{t=1}^{T}M_{t}\right]  \leq  \frac{1}{\gamma}\bar{L}_{\gamma,T}\left(\{u_{t}\}\right) \\
&+ \frac{a}{2\gamma^{2}}\left(b\left\Vert u_{1}\right\Vert ^{2}+cV_m +\mathbb{E}\left[\sum_{t=1}^{T}M_{t}Z_{t}\left(u_{t}^{\top}x_{t}\right)^{2}\right]\right) \\
& +\frac{1}{2a}\mathbb{E}\left[\sum_{t=1}^{T}M_{t}Z_{t}x_{t}^{\top}D_{t}^{-1}x_{t}\right]~.
\end{align*}
The value of the expected number of queried labels trivially follows, $\mathbb{E}\left[\sum_{t=1}^{T}Z_t\right]=\sum_{t=1}^{T}\mathbb{E}\left[\frac{a}{a+|\hat{p}_{t}|}\right]$~.
\QED
\end{proof}
Next, we further bound the term $\sum_{t\in\mathcal{M}}x_{t}^{\top}D_{k}^{-1}x_{t}$
in \thmref{thm:lasec_bound}. Using Lemma 5 and Lemma 7 of Moroshko and
Crammer~\cite{MoroshkoCr13} we have
\begin{align}
&\sum_{t\in\mathcal{M}}x_{t}^{\top}D_{k}^{-1}x_{t}  \leq  \ln\left|\frac{1}{b}D_{m}\right|+c^{-1}\sum_{k=1}^{m}\tr\left(D_{k-1}\right)\nonumber\\
 & \leq  \ln\left|\frac{1}{b}D_{m}\right|+c^{-1}\tr\left(D_{0}\right)\nonumber\\
 &+\frac{m}{c}d\max\left\{ \frac{3X^{2}+\sqrt{X^{4}+4X^{2}c}}{2},b+X^{2}\right\}~,
\label{a4}
\end{align}
where $\left\Vert x_{t}\right\Vert ^{2}\leq X^{2}$. Substituting \eqref{a4} in \eqref{laser_bound} we get a bound of the form
$m\leq\frac{1}{\gamma}D+\frac{1}{\gamma}\sqrt{A\left(B+mC\right)}$ for
the LASEC algorithm, solved for $m$ with the following technical lemma.
\begin{lemma}
\label{lem:lemma1}
Let $A,B,C,D,\gamma,m>0$ satisfy $m\leq\frac{1}{\gamma}D+\frac{1}{\gamma}\sqrt{A\left(B+mC\right)}$.
Then
\begin{align}
m\leq & \frac{1}{\gamma}D+\frac{1}{2\gamma^{2}}AC \nonumber \\
& +\frac{1}{\gamma}\sqrt{\frac{1}{\gamma}DAC+\frac{1}{4\gamma^{2}}\left(AC\right)^{2}+AB}~.
\label{new_bound}
\end{align}
\end{lemma}
The proof appears in the supplementary material.
Using \lemref{lem:lemma1} we have the bound \eqref{new_bound} for the LASEC algorithm,
where
\begin{align*}
&A = b\left\Vert u_{1}\right\Vert ^{2}+cV_m+\sum_{k=1}^{m}\left(u_{k}^{\top}x_{k}\right)^{2},\\
&B = \ln\left|\frac{1}{b}D_{m}\right|+c^{-1}\tr\left(D_{0}\right),\\
&C = c^{-1}d\max\left\{ \left(3X^{2}+\sqrt{X^{4}+4X^{2}c}\right)/2,b+X^{2}\right\},\\
&D = L_{\gamma,T}\left( \{u_t\}\right)~.
\end{align*}
Next, we use corollary 8 from Moroshko and Crammer~\cite{MoroshkoCr13} to get the final bound for LASEC.
\begin{corollary}
\label{cor1}
Assume $\left\Vert x_{t}\right\Vert ^{2}\leq X^{2}$ and
set $b=\varepsilon c$ for some $0<\varepsilon<1$. Denote $\mu=\max\left\{ 9/8X^{2},\frac{\left(b+X^{2}\right)^{2}}{8X^{2}}\right\}$. Assume the LASEC algorithm is run on $T$ examples. If
$V_m\leq T\frac{\sqrt{2}dX}{\mu^{3/2}}$ then by setting $c=\left(\frac{\sqrt{2}TdX}{V_m}\right)^{2/3}$
we have the bound \eqref{new_bound} for the number of mistakes of the LASEC algorithm,
 where
\begin{align*}
&D = L_{\gamma,T}\left( \{u_t\}\right)  , \\
&A = b\left\Vert u_{1}\right\Vert ^{2}+\left(\sqrt{2}dX\right)^{2/3}T^{2/3}V_{m}^{1/3}+\sum_{k=1}^{m}\left(u_{k}^{\top}x_{k}\right)^{2},\\
&B = \ln\left|\frac{1}{b}D_{m}\right|+\frac{\varepsilon}{1-\varepsilon}d  ,\\
&C = \left(4dX\right)^{2/3}T^{-1/3}V_{m}^{1/3}~.
\end{align*}
\end{corollary}
\begin{proof}
As was shown~\cite{MoroshkoCr13} we have
\[
\max\left\{ \frac{3X^{2}+\sqrt{X^{4}+4X^{2}c}}{2},b+X^{2}\right\} =2X\sqrt{2c}~,
\]
 and thus
\begin{align*}
A&=b\left\Vert u_{1}\right\Vert ^{2}+cV_m+\sum_{k=1}^{m}\left(u_{k}^{\top}x_{k}\right)^{2}\\
&=b\left\Vert u_{1}\right\Vert ^{2}+\left(\sqrt{2}dX\right)^{2/3}T^{2/3}V_{m}^{1/3}+\sum_{k=1}^{m}\left(u_{k}^{\top}x_{k}\right)^{2}~,\\
B&=\ln\left|\frac{1}{b}D_{m}\right|+c^{-1}\tr\left(D_{0}\right)=\ln\left|\frac{1}{b}D_{m}\right|+\frac{\varepsilon}{1-\varepsilon}d~,\\
C&=\frac{2\sqrt{2}dX}{\sqrt{c}}=\frac{2\sqrt{2}dX}{\left(\frac{\sqrt{2}TdX}{V_{m}}\right)^{1/3}}=\left(4dX\right)^{2/3}T^{-1/3}V_{m}^{1/3}~.
\end{align*}
\QED
\end{proof}
The last bound for the LASEC algorithm is difficult to interpret. Roughly speaking, the number of mistakes grows with the amount of drift as ${\sim}T^{1/3}V_{m}^{2/3}$, because $A{\sim}T^{2/3}V_{m}^{1/3}$, $C{\sim}T^{-1/3}V_{m}^{1/3}$ and the bound is ${\sim}AC$.  Another bound for the drifting setting was shown by Cavallanti et al for the Shifting Perceptron~\cite{CavallantiCG07}. However, they used other notation of drift, which uses the norm rather than the square norm of the difference of comparison vectors, as we do. Thus, the two bounds are not comparable in general.

Next, we move to get explicit mistake bound for the LASEC-SS algorithm, by bounding the right term in \thmref{thm:lasec_ss_bound}.
Again, using Lemma 5 and Lemma 7 from \cite{MoroshkoCr13} we get,
\begin{align*}
&\sum_{t=1}^{T}M_{t}Z_{t}x_{t}^{\top}D_{t}^{-1}x_{t}  \leq  \ln\left|\frac{1}{b}D_{T}\right|+c^{-1}\sum_{t=1}^{T}\tr\left(D_{t-1}\right) \\
& \leq  \ln\left|\frac{1}{b}D_{T}\right|+c^{-1}\tr\left(D_{0}\right) \\
& +Tc^{-1}d\max\left\{ \left(3X^{2}+\sqrt{X^{4}+4X^{2}c}\right)/2,b+X^{2}\right\}~.
\end{align*}
Combining this bound with \thmref{thm:lasec_ss_bound} we get,
\begin{align*}
&\mathbb{E}\left[\sum_{t=1}^{T}M_{t}\right]  \leq \frac{1}{\gamma}\bar{L}_{\gamma,T}\left(\{u_{t}\}\right) \\
&+\frac{a}{2\gamma^{2}}\left(b\left\Vert u_{1}\right\Vert ^{2}+cV_{m}+\mathbb{E}\left[\sum_{t=1}^{T}M_{t}Z_{t}\left(u_{t}^{\top}x_{t}\right)^{2}\right]\right) \\
&  +\frac{1}{2a}\Bigg(\mathbb{E}\ln\left|\frac{1}{b}D_{T}\right|+c^{-1}\tr\left(D_{0}\right) \\
&+Tc^{-1}d\max\left\{ \left(3X^{2}+\sqrt{X^{4}+4X^{2}c}\right)/2,b+X^{2}\right\} \Bigg).
\end{align*}
We now state the main result of this section, bounding the expect number of mistakes of the LASEC-SS algorithm. This is an immediate application of corollary 8 from \cite{MoroshkoCr13}.
\begin{corollary}
Assume $\left\Vert x_{t}\right\Vert ^{2}\leq X^{2}$ and
set $b=\varepsilon c$ for some $0<\varepsilon<1$. Denote $\mu=\max\left\{ 9/8X^{2},\frac{\left(b+X^{2}\right)^{2}}{8X^{2}}\right\}$. Assume the LASEC-SS algorithm is run on $T$ examples. If
$V_{m}\leq T\frac{\sqrt{2}dX}{\mu^{3/2}}$ then by setting $c=\left(\frac{\sqrt{2}TdX}{V_{m}}\right)^{2/3}$
we get
\begin{align*}
&\mathbb{E}\left[\sum_{t=1}^{T}M_{t}\right]  \leq \frac{1}{\gamma}\bar{L}_{\gamma,T}\left(\{u_{t}\}\right)+b\frac{a}{2\gamma^{2}}\left\Vert u_{1}\right\Vert ^{2}\\
&+\frac{a}{2\gamma^{2}}\left(\sqrt{2}dX\right)^{2/3}T^{2/3}V_{m}^{1/3}\\
&+\frac{a}{2\gamma^{2}}\mathbb{E}\left[\sum_{t=1}^{T}M_{t}Z_{t}\left(u_{t}^{\top}x_{t}\right)^{2}\right]\\
&+\frac{1}{2a}\left(\mathbb{E}\ln\left|\frac{1}{b}D_{T}\right|+\frac{\varepsilon}{1-\varepsilon}d+\left(4dX\right)^{2/3}T^{2/3}V_{m}^{1/3}\right)~.
\end{align*}
\end{corollary}
Again, we can optimize the last bound for the algorithm's parameter $a$. Setting
\begin{equation*}
a=\gamma\sqrt{
\frac{\mathbb{E}\ln\left|\frac{1}{b}D_{T}\right|+\frac{\varepsilon}{1-\varepsilon}d+\left(4dX\right)^{2/3}T^{2/3}V_{m}^{1/3}}
{\parbox{5cm}{$\Big(b\left\Vert u_{1}\right\Vert ^{2}+\left(\sqrt{2}dX\right)^{2/3}T^{2/3}V_{m}^{1/3}$\\
\hspace*{1cm}$+\mathbb{E}\left[\sum_{t=1}^{T}M_{t}Z_{t}\left(u_{t}^{\top}x_{t}\right)^{2}\right]\Big)$}}
}
\end{equation*}
we obtain
\begin{align*}
&\mathbb{E}\left[\sum_{t=1}^{T}M_{t}\right]  \leq  \frac{1}{\gamma}\bar{L}_{\gamma,T}\left(\{u_{t}\}\right)\\
&  +\frac{1}{\gamma}\Bigg[\bigg(b\left\Vert u_{1}\right\Vert ^{2}+\left(\sqrt{2}dX\right)^{2/3}T^{2/3}V_{m}^{1/3}\\
& ~~~~ +\mathbb{E}\left[\sum_{t=1}^{T}M_{t}Z_{t}\left(u_{t}^{\top}x_{t}\right)^{2}\right]\bigg)
\bigg(\mathbb{E}\ln\left|\frac{1}{b}D_{T}\right|+\frac{\varepsilon}{1-\varepsilon}d\\
& ~~~~ +\left(4dX\right)^{2/3}T^{2/3}V_{m}^{1/3}\bigg)\Bigg]^{1/2}~.
\end{align*}

\section{Experimental Study}
\begin{figure*}[t!]
\subfigure[\label{fig:sims_full}]{\includegraphics[width=0.33\textwidth]{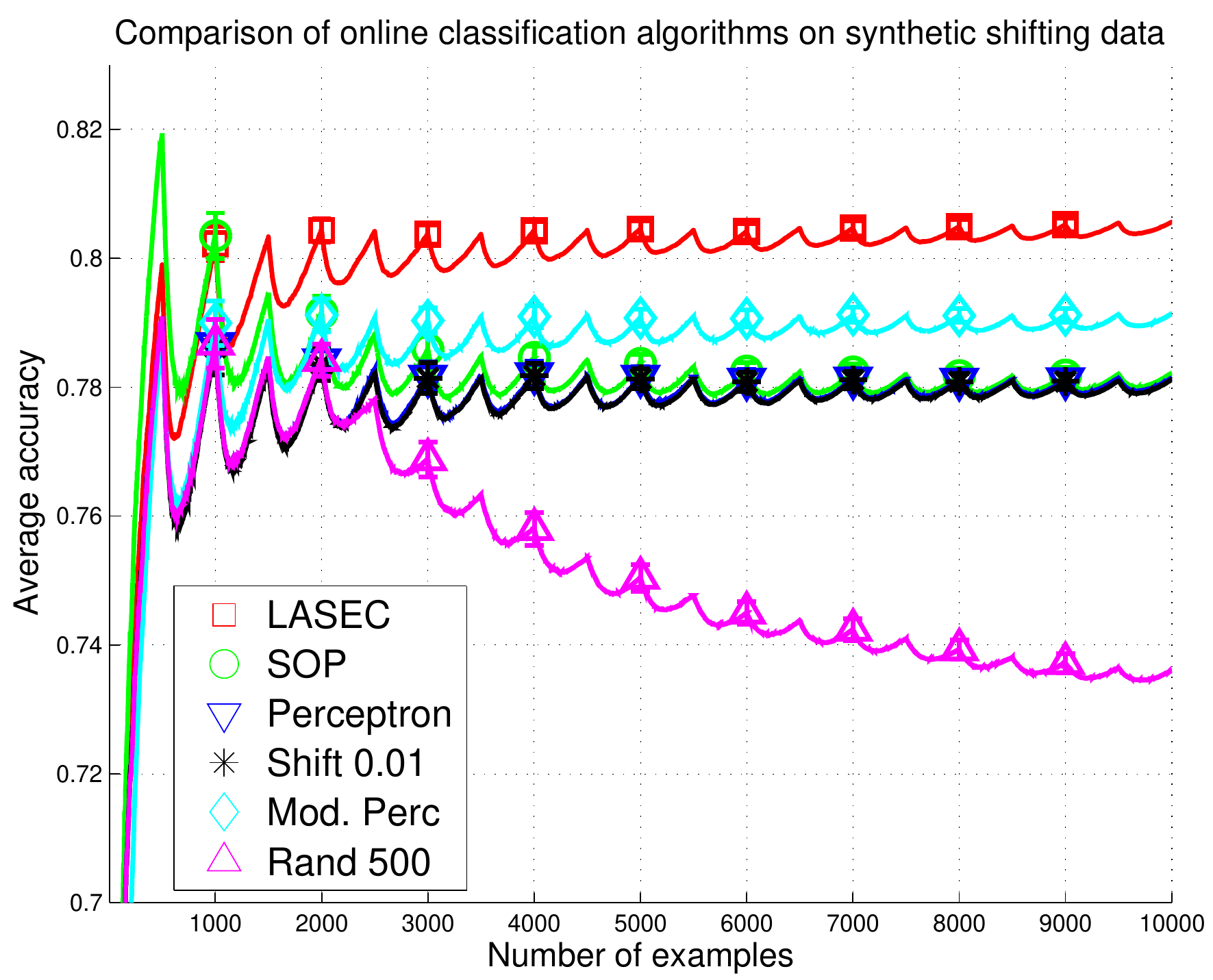}}
\subfigure[\label{fig:sims_ss_single_1}]{\includegraphics[width=0.33\textwidth]{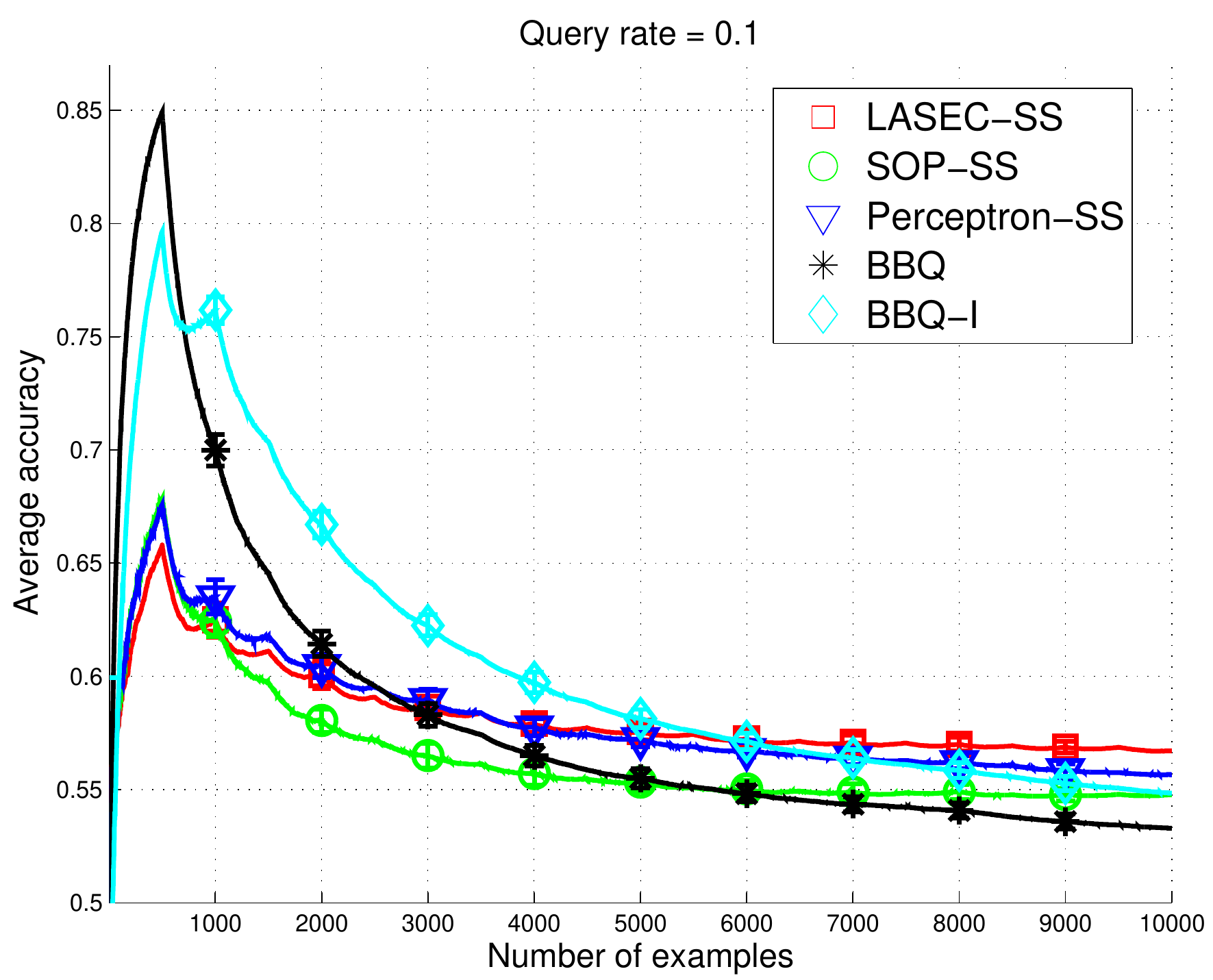}}
\subfigure[\label{fig:sims_ss_all}]{\includegraphics[width=0.33\textwidth]{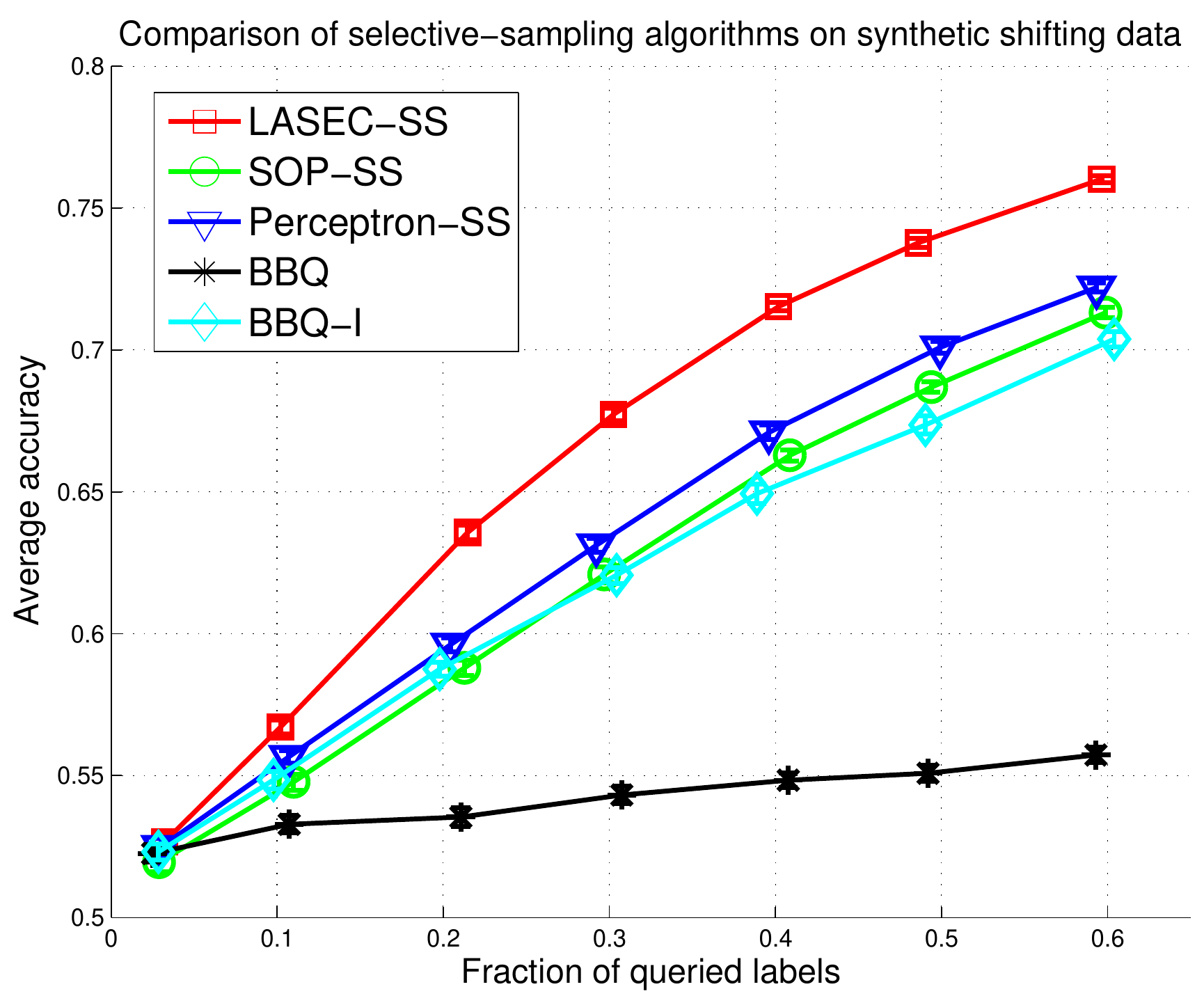}}
\subfigure[\label{fig:sims_full_real}]{\includegraphics[width=0.33\textwidth]{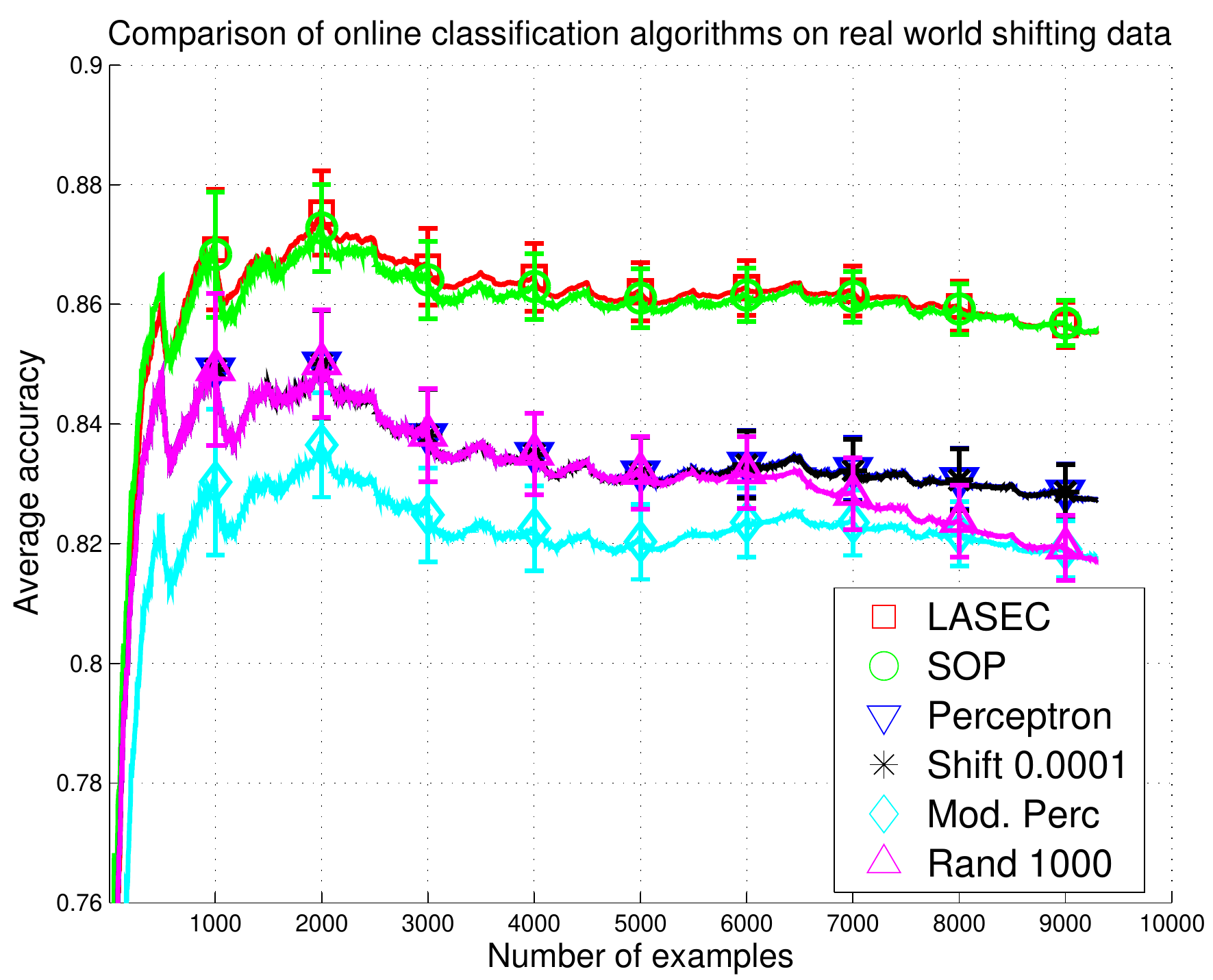}}
\subfigure[\label{fig:sims_ss_single_4}]{\includegraphics[width=0.33\textwidth]{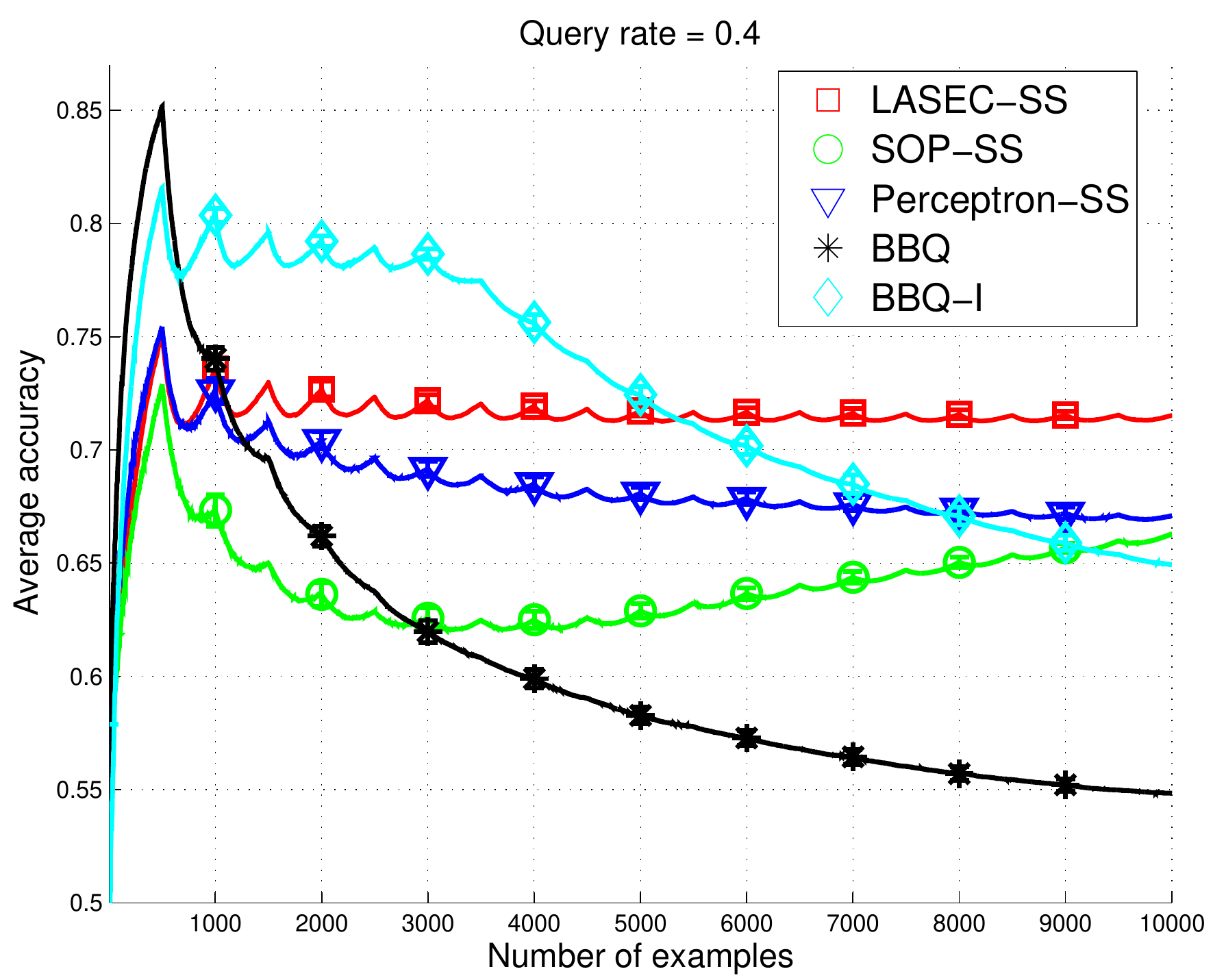}}
\subfigure[\label{fig:sims_ss_all_real}]{\includegraphics[width=0.33\textwidth]{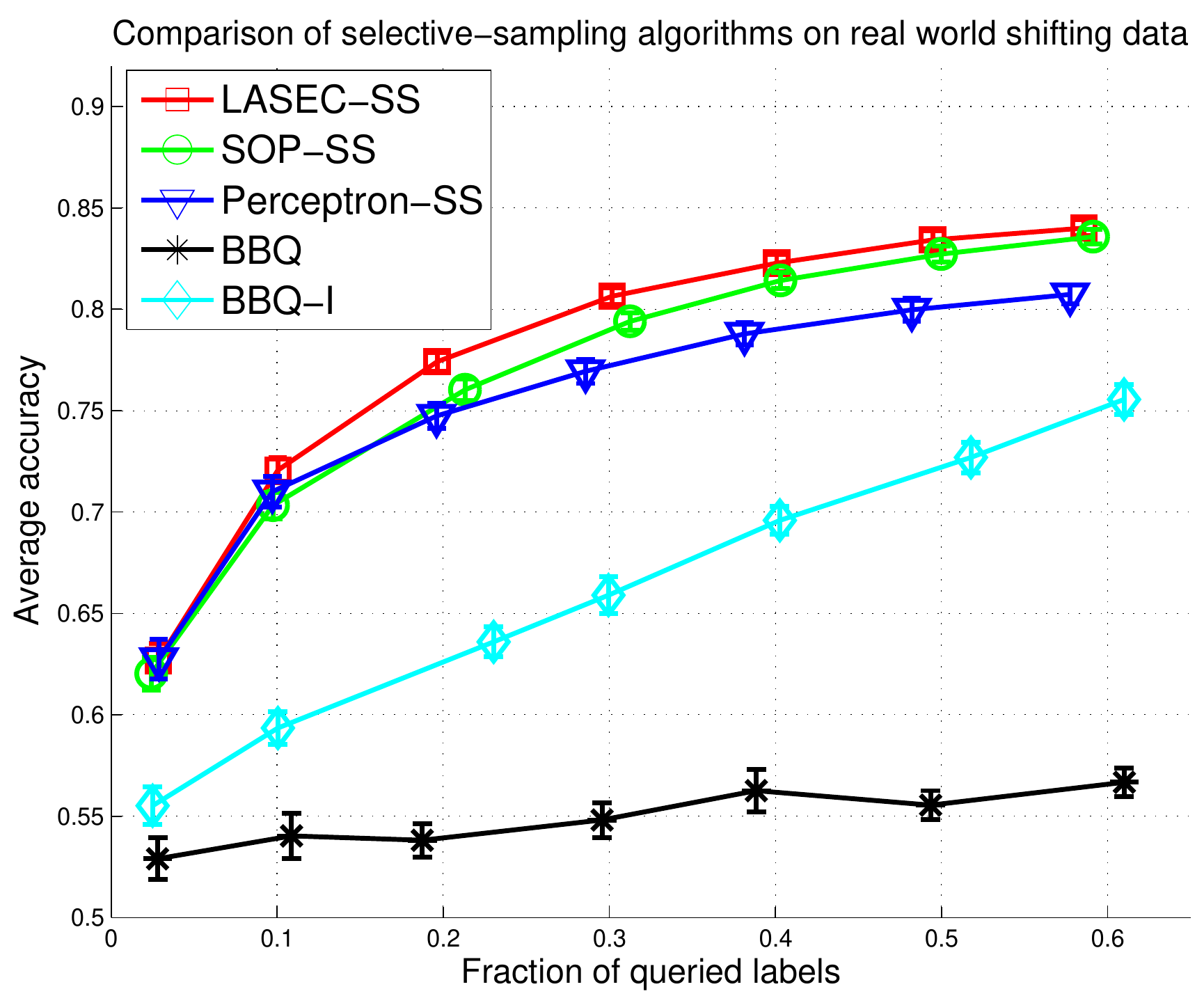}}
\caption{Left: accuracy against number of examples for binary fully supervised classification algorithms on (a) synthetic shifting dataset, and (d) USPS shifting dataset.
Middle: accuracy against number of examples with query rate (b) $\mathord{\sim}0.1$, and (e) $\mathord{\sim}0.4$.
Right: accuracy against fraction of queried labels for selective sampling algorithms on (c) synthetic shifting dataset, and (f) USPS shifting dataset.}
\vspace{-0.2cm}
\end{figure*}
%

We evaluated our algorithm with both synthetic and real-world data with shifts, by comparing the average accuracy (total number of correct online classifications divided
by the number of examples) of LASEC and LASEC-SS.
\paragraph{Data: }
In our first experiment we use a synthetic dataset with $10,000$ examples of dimension $d=50$. The inputs $x_t\in\reals^{50}$ were drawn from a zero-mean unit-covariance Gaussian distribution. The target $u_t\in\reals^{50}$ is a zero-mean unit-covariance Gaussian vector, which is switched every 500 examples to some other random vector. That is $u_1=...=u_{500},~u_{501}=...=u_{1000},~...$. The labels are set according to $y_t=\sign(x_t^{\top}u_t)$.
Our second experiment uses the US Postal Service handwritten digits recognition corpus (USPS)~\cite{Hull:1994:DHT:628312.628607}. It contains normalized
grey scale images of size $16{\times} 16$, divided into a training (test) set of $7,291$ ($2,007$) images. We combined the sets to get $9,298$ examples.  Based on the USPS multiclass data we generated binary data with shifts. We chose at random some digits to be positive class (the other digits are negative class). The partition to positive and negative classes is changed every 500 examples at random, that is every 500 samples we changed the subgroup of labels (out of 10) that are labeled as +1 (labels in the complementary subgroup are labeled as -1).
Each set of experiments was repeated $50$ times and the error bars in the plots correspond to the $95\%$ confidence interval over the $50$ runs.
\paragraph{Supervised Online Learning with Drift: }
In the supervised-online classification task we compared the performance of LASEC (setting $a=\infty$ in \figref{algorithm:ss_lasec}) to five other algorithms: the second-order perceptron algorithm (SOP) \cite{CesaBianchiCoGe05}, the Perceptron algorithm \cite{Rosenblatt58}, the Shifting Perceptron algorithm \cite{CavallantiCG07}, the Modified Perceptron algorithm \cite{DBLP:journals/algorithmica/BlumFKV98} and the Randomized Budget Perceptron algorithm \cite{CavallantiCG07}.

Both the Shifting Perceptron and the Randomized Budget Perceptron are tuned using a single
parameter (denoted by $\lambda$ and B respectively). Since the optimal values of these parameters simply reduced these algorithms to the original Perceptron, we set $\lambda=0.01$ and $B=500$ for synthetic data, and $\lambda=0.0001$ and $B=1,000$ for real-world data. The setting of $B=500$ is actually the switching window, while for real-world data we alleviated the Randomized Budget Perceptron and used twice the switching window as the budget.
For the LASEC and SOP algorithms the parameters were tuned using a random draw of the data.

The results comparing supervised classification algorithms on synthetic data are shown in \figref{fig:sims_full}. While for $t<500$ (before the first shift) the SOP algorithm is the best as expected, we see that the LASEC algorithm deals better with the shifts and outperforms other algorithms.
For the real-world USPS dataset (see \figref{fig:sims_full_real})
LASEC slightly outperforms SOP, and both outperform other
perceptron-like algorithms, due to the usage of second-order
information.

\paragraph{Selective Sampling Online Learning with Drift: }
For the selective sampling task we compared the LASEC-SS algorithm from \figref{algorithm:ss_lasec} to several selective sampling algorithms: the selective sampling version of second-order perceptron algorithm (SOP-SS) \cite{Cesa-BianchiGZ06a}, the selective sampling version of perceptron algorithm (Perceptron-SS) \cite{Cesa-BianchiGZ06a} and the BBQ algorithm \cite{DBLP:conf/icml/Cesa-BianchiGO09}. In addition to the BBQ algorithm~\cite{DBLP:conf/icml/Cesa-BianchiGO09}, we also consider a variant of the BBQ algorithm, which we call BBQ-I. This algorithm is similar to the original BBQ algorithm but it performs updates only when the queried label is different from the predicted label.
Each algorithm has one parameter that controls the tradeoff between the query rate (fraction of queried labels) and the accuracy of the algorithm. 
For fairness, this parameter was set to get about the same query rate for all algorithms.

\figref{fig:sims_ss_single_1} and \figref{fig:sims_ss_single_4} summarize the accuracy of the algorithms on synthetic data for query rate $\mathord{\sim}0.1$ and $\mathord{\sim}0.4$. In both cases LASEC-SS outperforms other algorithms. Before the first shift at round 500 the BBQ is the best as expected from previous results \cite{DBLP:conf/icml/OrabonaC11}, but its performance significantly degrade after the first shift. This is because this algorithm performs query when the quantity $r_t=x_t^{\top}A_t^{-1}x_t$ is large enough (see \cite{DBLP:conf/icml/Cesa-BianchiGO09}), while the matrix $A_t$ grows each time a label is queried. Large $A_t$ makes $r_t$ small, the algorithm converges and stops query labels. If after that a switch occurs, the algorithm fails in predictions but cannot query correct labels because $r_t$ is small. This causes a significant degradation in the prediction accuracy. On the other hand, the BBQ-I algorithm performs less updates (as it performs updates only when a mistake occurs), and thus the algorithm converges much slower. This makes it simpler to adapt to changing environment, after a switch occurs. We note that for stationary environment (when $u_t=u~\forall t$), BBQ outperforms BBQ-I, as well as other selective sampling algorithms (see \cite{DBLP:conf/icml/OrabonaC11}). For low query rate as in \figref{fig:sims_ss_single_1}, all algorithms hardly deal with the shifts, as expected. However, our algorithm still converges on a better average accuracy. For higher rate as in \figref{fig:sims_ss_single_4}, our algorithm deals well with the shifts and the average accuracy does not decrease. This is in contrary to other algorithms. For the SOP-SS algorithm we see in \figref{fig:sims_ss_single_4} that the performance increase over time after an initial drop. This is because the SOP algorithm tends to converge 
fast and then it acts as no labels are needed, because the margin is large. After a data shift, the algorithm experiences a drop because no labels are sampled. After some time the algorithm detects that labels are needed (because the margin is small) and performance increase.

\figref{fig:sims_ss_all} and \figref{fig:sims_ss_all_real} show the tradeoff between average accuracy and fraction of queried labels on synthetic and real-world (USPS) data accordingly. Evidently, LASEC-SS is the best selective sampling algorithm in the drifting setting. In addition, we can see that unlike the stationary setting where it was shown \cite{Cesa-BianchiGZ06a} that a small fraction of labels are enough to get the accuracy of a fully supervised setting, in the drifting case much more labels are needed. This is because old queried labels cannot contribute to form a good predictor due to a drift in the model, and the algorithm must query more labels to have a good prediction accuracy. Yet, our algorithm can employ half of the labels to get performance not too far from the full information case.




\vspace{-0.2cm}
\section{Conclusions}
\vspace{-0.2cm}
We proposed a novel second-order algorithm for binary classification designed to work in non-stationary (drifting) selective sampling setting. Our algorithm is based on the last-step min-max approach, and we showed how to solve the last-step min-max optimization problem directly for classification using the square loss. 
To the best of our knowledge, this is the first algorithm designed to work in the selective sampling setting when there is a drift. We proved mistake bound for the algorithm in the fully supervised setting, and a bound for the expected number of mistakes for the selective sampling version of the algorithm. Experimental study shows that our algorithm outperforms other algorithms, in the supervised and selecting sampling settings.
%
For the algorithm to perform well, the amount of drift $V$ or a bound over it should be known in advance. An interesting direction is to design algorithms that automatically detect the level of drift, or are invariant to it.

\paragraph{Acknowledgements:} This research was funded in part by the Intel Collaborative Research Institute for Computational Intelligence (ICRI-CI) and in part by an Israeli Science Foundation grant ISF- 1567/10.

{
\bibliographystyle{plain}
\bibliography{bib}
}
\newpage
\appendix
\section{SUPPLEMENTARY MATERIAL}
\label{sec:supp_material}


\subsection{Proof of \lemref{lem:lemma1}}
\label{proof_lemma1}
\begin{proof}
We follow equivalency of the following inequalities,
\begin{align*}
&m-\frac{1}{\gamma}D\leq\frac{1}{\gamma}\sqrt{A\left(B+mC\right)}\\
&m^{2}-\frac{2}{\gamma}mD+\frac{1}{\gamma^{2}}D^{2}\leq\frac{1}{\gamma^{2}}A\left(B+mC\right)\\
&m^{2}-\left(\frac{2}{\gamma}D+\frac{1}{\gamma^{2}}AC\right)m+\frac{1}{\gamma^{2}}D^{2}-\frac{1}{\gamma^{2}}AB\leq0\\
&
m\leq\Bigg(\frac{2}{\gamma}D+\frac{1}{\gamma^{2}}AC\\
&~~~+\sqrt{\left(\frac{2}{\gamma}D+\frac{1}{\gamma^{2}}AC\right)^{2}-4\left(\frac{1}{\gamma^{2}}D^{2}-\frac{1}{\gamma^{2}}AB\right)}\Bigg)/2\\
&=\frac{1}{\gamma}D+\frac{1}{2\gamma^{2}}AC\\
&~~~+\frac{1}{\gamma}\sqrt{\frac{1}{\gamma}DAC+\frac{1}{4\gamma^{2}}\left(AC\right)^{2}+AB}~.
\end{align*}
\QED
\end{proof}


\end{document}